%% file: main.tex
\documentclass{article}

\usepackage{microtype}
\usepackage{graphicx}
\usepackage{subfigure}
\usepackage{multirow}
\usepackage{color}
\usepackage{booktabs} 
\usepackage[algo2e]{algorithm2e} 
\usepackage{pdfpages}

\usepackage{hyperref}

\usepackage[accepted]{icml2024}

\usepackage{amsmath}
\usepackage{amssymb}
\usepackage{mathtools}
\usepackage{amsthm}
\usepackage{enumitem}

\usepackage[capitalize,noabbrev]{cleveref}

\newcommand{\vQ}{\mathbf{Q}}
\newcommand{\vK}{\mathbf{K}}
\newcommand{\vV}{\mathbf{V}}

\newcommand{\vS}{\mathbf{S}}

\newcommand{\vP}{\mathbf{P}}

\newcommand{\vO}{\mathbf{O}}

\newcommand{\diag}{\mathrm{diag}}

\newcommand{\flashattn}{FlashAttention\xspace}

\theoremstyle{plain}
\newtheorem{theorem}{Theorem}[section]
\newtheorem{proposition}[theorem]{Proposition}

\theoremstyle{definition}
\newtheorem{definition}[theorem]{Definition}

\theoremstyle{remark}

\icmltitlerunning{Preprint}

\begin{document}

\twocolumn[
\icmltitle{StableMask: Refining Causal Masking in Decoder-only Transformer}



\icmlsetsymbol{equal}{*}
\begin{icmlauthorlist}
\icmlauthor{Qingyu Yin}{zju}
\icmlauthor{Xuzheng He}{pku}
\icmlauthor{Xiang Zhuang}{zju}
\icmlauthor{Yu Zhao}{tencent}
\icmlauthor{Jianhua Yao}{tencent}
\icmlauthor{Xiaoyu Shen}{eit}
\icmlauthor{Qiang Zhang}{zju}

\end{icmlauthorlist}
\icmlaffiliation{zju}{Zhejiang University}
\icmlaffiliation{pku}{Peking University}
\icmlaffiliation{tencent}{Tencent AI Lab}
\icmlaffiliation{eit}{Eastern Institute of Technology, Ningbo}

\icmlcorrespondingauthor{Qiang Zhang}{qiang.zhang.cs@zju.edu.cn}
\icmlcorrespondingauthor{Xiaoyu Shen}{xyshen@eitech.edu.cn}

\icmlkeywords{Machine Learning, ICML}

\vskip 0.3in
]



\printAffiliationsAndNotice{}

\begin{abstract}
The decoder-only Transformer architecture with causal masking and relative position encoding (RPE) has become the \emph{de facto} choice in language modeling. Despite its exceptional performance across various tasks, we have identified two limitations: First, it requires all attention scores to be non-zero and sum up to 1, even if the current embedding has sufficient self-contained information. This compels the model to assign disproportional excessive attention to specific tokens. Second, RPE-based Transformers are not universal approximators due to their limited capacity at encoding absolute positional information, which limits their application in position-critical tasks. In this work, we propose \emph{StableMask}: a parameter-free method to address both limitations by refining the causal mask. It introduces pseudo-attention values to balance attention distributions and encodes absolute positional information via a progressively decreasing mask ratio. StableMask's effectiveness is validated both theoretically and empirically, showing significant enhancements in language models with parameter sizes ranging from 71M to 1.4B across diverse datasets and encoding methods. We further show that it naturally supports (1) efficient extrapolation without special tricks such as StreamingLLM and (2) easy integration with existing attention optimization techniques.
\end{abstract}

\input{section/1_introduction.tex}

\input{section/2_Over_Attention_and_PE}

\input{section/3_StableMask}

\input{section/4_Experiment}
\input{section/5_Conclusion}

\newpage
\bibliography{citation}
\bibliographystyle{icml2024}

\newpage
\appendix
\onecolumn

\input{section/6_Appendix}

\end{document}

%% file: section/1_introduction.tex
\section{Introduction}
Large Language Models (LLMs) have revolutionized natural language processing for their task-agnostic in-context learning paradigm~\cite{brown2020language}. The core of LLMs is the decoder-only Transformer architecture~\cite{vaswani2017attention,radford2019language}, characterized by the self-attention mechanism and relative positional encoding (RPE) to aggregate information and catch the dependency among tokens. It has exhibited superior zero-shot generalization capabilities in comparison to its encoder-decoder counterparts, leading to its increased prevalence in pre-trained LLMs~\cite{lester2021power,patel2023bidirectional}. Despite the impressive success, we identified two important issues within this architecture.

The first issue arises from the softmax function used in self-attention, as its outputs consist solely of non-zero values summing up to 1~\cite{pang2019rethinking}. This forces to allocate a certain distribution of attention probability across all available tokens, even when the current token already has sufficient self-contained information~\cite{xiao2023efficient} or when the attention mechanism does not need to prioritize any token~\cite{hua2022Transformer,bondarenko2023quantizable}. In such cases, the model tends to allocate \textit{disproportional attention} scores to specific tokens like punctuation marks. This problem is exacerbated in decoder-only models as the varied sequence length leads to an extremely uneven attention distribution, particularly on the initial tokens. While approaches have been proposed to mitigate this issue, they all entail significant complexity. e.g., modifying the sparseness of softmax~\cite{laha2018controllable}, or adding dedicated tokens to absorb unnecessary attention~\cite{darcet2023vision}.

The second limitation is associated with various relative positional encoding strategies~\cite{ke2020rethinking}, e.g. ALiBi~\cite{press2022train}, T5~\cite{raffel2020exploring}, and RoPE~\cite{su2021roformer}. Compared with absolute position encoding (APE), RPE has achieved state-of-the-art performance in most natural language task. It also exhibits better extrapolation capabilities, and naturally preserves invariant properties
for several important transformations like rotation and translation, making it more widely used in Transformers~\cite{press2022train}. However,  RPE fails to capture enough absolute positional information as the softmax always generates a {right stochastic matrix}~\cite{luo2022your}, i.e., a square matrix where each row consists of non-negative real numbers adding up to 1. 
This restricts its application in situations where such positional information is crucial. Previous attempts to address this, such as URPE~\cite{luo2022your}, added learnable relative position matrices atop the softmax outputs, which hurt the extrapolation capabilities because of the non-extensibility of learnable parameters.

In this paper, we propose \emph{StableMask} – a tailored approach to address both issues by carefully modifying the causal mask in the decoder-based transformers.
It introduces extra pseudo attention scores to the upper triangular attention matrix, which stabilizes the normalization constant of attention scores within each row regardless of the sequence length and token position. This 
allows the model to allocate excess attention to these dedicated pseudo scores. Moreover, StableMask progressively ensures that the result of softmax is not a right stochastic matrix. With a decreasing mask ratio (i.e. the sum of each row after softmax), it enables the model to encode a measurement of absolute position during the softmax stage, while remaining consistent with the decaying inter-token dependency used in RPE, thus effectively maintaining its extrapolation capability.


StableMask's effectiveness has been thoroughly validated through extensive testing on multiple language models across a diverse array of both synthetic and realistic tasks. 
It represents a substantial advancement in refining the attention mechanisms for decoder-only Transformers, overcoming the inherent limitations while retaining their core strengths. 
A key advantage of StableMask is its parameter-free nature. As StableMask is implemented solely as a direct replacement for the causal mask, it is highly compatibile with the Transformer's native architecture (such as different position encodings, attention optimizations or extrapolation techniques). For instance, we have presented an implementation of StableMask that is optimized for hardware efficiency, aligning with the principles of \flashattn~\cite{dao2022flashattention}. This allows StableMask to seamlessly integrate into the ecosystem of Transformer models, thereby expanding its potential applications.

Our core contributions can be summarized as follows:
\begin{enumerate}
    \item We identified two issues in the commonly used decoder-only Transformer architecture: the disproportional attention distribution and the inability to accurately capture positional information.
    \item We propose {StableMask}, an efficient and easily integrable solution to effectively address both issues by carefully modifying the causal mask.
    \item We validate the effectiveness of StableMask across multiple tasks and encoding methods.
    \item We present a hardware-efficient version of StableMask to optimize its practical applicability.
\end{enumerate}

%% file: section/2_Over_Attention_and_PE.tex
\begin{figure*}[!ht]
    \centering
    \includegraphics[width=1\textwidth]{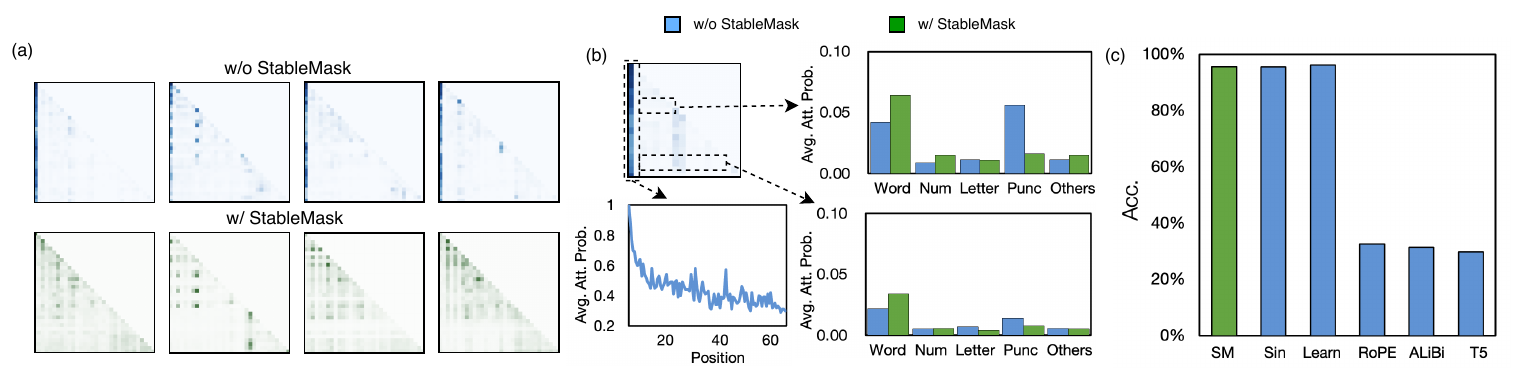}
    \vspace{-0.5cm} 
    \caption{(a) Visual comparison of attention heads with and without StableMask on the OpenLLaMA 1.4B model. (b) The attention allocation to various types of tokens (excluding the initial token) at two different positions and the trend of attention allocation to the initial token over positions, \textit{averaged over heads}. Blue: The original Transformer exhibits a clear disproportional attention issue. Green: StableMask effectively rectifies the proportion of attention allocation. (c) Experimental Results showing RPE’s inability to encode absolute position (Blue). StableMask solves the issue of RPE’s inability to encode absolute position (Green). }
    \label{fig:da_tokens}
\end{figure*}

\section{Preliminary}
\paragraph{Self-Attention}
Let $X$ be the input sequence, $n$ be the sequence length and \(d\) be the dimensionality of the hidden state.
The self-attention mechanism in Transformer architectures calculates attention scores between each pair of words to capture dependencies between words and learn contextual information effectively.
Let \(A\) denote the attention score matrix and \(a_{ij}\) be the attention score between the \(i\)-th word and the \(j\)-th word. We have
    \(A = \frac{QK^\top}{\sqrt{d}}\)where \( Q, K, V \in \mathbb{R}^{n \times d}\) represent the Query, Key, and Value matrices derived from \(X\)~\cite{vaswani2017attention}.  In decoder-only models, \(A\) is further modified by a causal mask \(M\) and a softmax operation:
\begin{equation}
    \tilde{A} = \mathrm{Softmax}(A + M).
    \label{eq:causal_softmax}
\end{equation}
The following holds to prevent the model from attending to future tokens:
\begin{eqnarray}
    M_i &=& [\underbrace{0, \cdots, 0}_{i}, \underbrace{-\infty, \cdots, -\infty}_{n-i}]_n, \\
    \tilde{A_i} &=& [a_{i1}, a_{i2}, \cdots, a_{ii}, 0, \cdots, 0]_n.
\end{eqnarray}

\paragraph{Position Encoding}
The raw Transformer without position encodings is insensitive to permutational rearrangements.
Two chief methods have been employed to remove this insensitivity: absolute position encoding (APE) and relative position encoding (RPE). APE assigns an index-dependent vector at each position to the word embeddings. These assigned vectors are usually trainable parameters to represent absolute positions of each input token~\cite{kenton2019bert,radford2019language}.
More recently, RPE such as ALiBi~\cite{press2022train}, T5~\cite{raffel2020exploring} and RoPE~\cite{su2021roformer} took a different approach by incorporating relative distances of positions into the attention score matrix. RPEs can be mainly classified into additive (T5, ALiBi, etc.) or multiplicative (RoPE, etc.):
\begin{align}
    \text{Add:}&& \tilde{A}_{\mathrm{add}} &= \mathrm{Softmax}\left(\frac{QK^\top + S}{\sqrt{d_k}} +M \right) , \\
    \text{Mul:} && \tilde{A}_{\mathrm{mul}} &= \mathrm{Softmax}\left(\frac{\tilde{Q} \tilde{K}^\top}{\sqrt{d_k}} +M \right) ,
\end{align}
where $\tilde{Q} = Q \odot R_{Q}, \ \tilde{K} = K \odot R_{K}$. $R_Q, R_K $ are rotary forms usually in complex values and \(S\) is a Topelitz matrix. Given its consistent demonstrated improvements over APE, RPE has emerged as the default choice in LLMs.
\section{Problem}
\label{sec:two}
Despite the exceptional performance, we identified two key issues associated with self-attention and RPE.
\paragraph{Disproportional Attention}
\label{sec:da}
The first issue arises from the softmax function used in
self-attention. Given that the softmax function requires {all attention scores to be non-zero and sum up to 1}, it necessitates an inescapable distribution of attention across {on all visible tokens}. However, previous studies~\cite{shen2019improving,hassid2022does,bondarenko2023quantizable,xiao2023efficient} have shown that the attention mechanism often requires very few important tokens, and the others are merely distractions. In this case, the requirement imposed by the softmax function prevents the model from effectively zeroing out the attention scores for irrelevant tokens. Some of these irrelevant tokens, such as initial tokens or non-functional words like punctuation marks, are more frequently observed by other tokens. In consequence, as shown in Figure~\ref{fig:da_tokens}, the model tends to allocate disproportional attention (DA) to them. We refer to these tokens which are not semantically relevant, but receive disproportional attention values, as DA tokens\footnote{Appendix~\ref{appendix:dispro} offers an information-theoretic definition and interpretation of the DA issue.}. The existence of DA tokens can lead to various undesired problems, e.g., perplexity surge in length extrapolation or sensitivity to irrelevant noise~\cite{xiao2023efficient}.



Interestingly, the extent of this DA phenomenon varies across token positions within the decoder-only language model. It is most prominent at the beginning of a sequence, and gradually eases towards the end {(as seen in Figure~\ref{fig:da_tokens}(b))}. Intuitively, as the token position increases, more tokens participate in the softmax operation and even assigning a very small probability to each token can result in a significant accumulative probability. As a result, DA tokens cannot receive as much attention values as they do near the beginning of a sequence. 



Existing solutions, such as StreamingLLM~\cite{xiao2023efficient} and ViT Register~\cite{darcet2023vision}, have attempted to address this by introducing \emph{Artificial Tokens} (AT) to absorb excess attention, so that real tokens can be freed from getting unnecessary DA. We term them as AT-based methods. However, as said, the severity of the DA issue varies along token positions. We hypothesize that adding a fixed number of tokens across all sequences is not position-adaptive and thereby cannot fully address the DA issue.

\paragraph{Inability to Encode Absolute Position}
\label{sec:inability}
Despite its superior performance, RPE that modifies \(QK^\top\) does not ensure \(V\) is also sensitive to position. For instance, when all inputs are identical vectors, the outputs are also guaranteed to be equal because the output of softmax generates a right stochastic matrix~\footnote{For a more in-depth discussion on all-identical inputs and their relation to DA, refer to Appendix~\ref{appendix:edge-case}.}. Therefore, RPE can perform poorly in tasks where positional information is critical. 

To verify this limitation of RPEs, we designed specialized datasets, inspired by URPE~\cite{luo2022your}, which focus on tasks requiring absolute positional information while maintaining consistent input sequences (check Appendix~\ref{sec:experiment_rpe} for details). We report {the average accuracy of various models in Figure~\ref{fig:da_tokens}(c)}. The results demonstrate that models relying exclusively on RPEs exhibit poor performance, confirming the inferiority of RPE in capturing absolute positional information.

\begin{figure*}[!htp]
    \centering
    \includegraphics[width=1\textwidth]{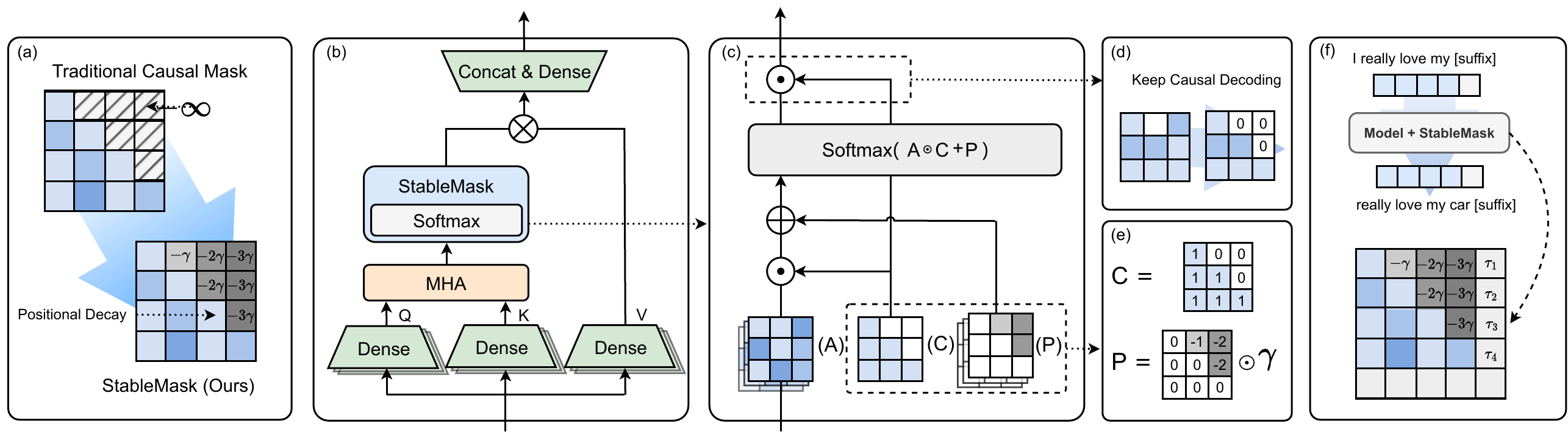}
    \vspace{-0.3cm} 
    \caption{(a) Illustration of the StableMask mechanism. (b) StableMask integrates with the softmax operation, replacing the traditional causal mask. (c) The attention score matrix is first cleared of attention values in the upper triangular part using the $C$ matrix, then pseudo-attention scores are added using the $P$ matrix followed by the softmax computation. (d) After the softmax operation, the remaining attention probabilities in the upper triangular part are cleared using $C$ to ensure the causal decoding property. (e) The $C$ matrix has zeros in the upper triangular part and ones in the lower triangular part, while the $P$ matrix has linear decay in the upper triangular part and zeros in the lower triangular part. $\gamma$ is a hyperparameter. (f) StableMask for inference. An input sequence needs a suffix.} 
    \label{fig:stablemask_arch}
\end{figure*}

One obvious solution to the limitation is to directly replace RPE with APE. However, as mentioned, APE has its own problems such as poor extrapolation, rotation and translation variant, worse prediction accuracy, etc~\cite{su2021roformer,press2022train}.
Another approach is to add additional parameters to the matrix after the softmax to re-encode absolute positional information. For example, URPE~\cite{luo2022your}
adds a learnable Toeplitz matrix $\mathcal{T}$ to the softmax matrix $\tilde{A}$ via:
\begin{eqnarray}
    \mathrm{Attention}(Q, K, V) = (\tilde{A} \odot \mathcal{T})V.
\end{eqnarray}
The URPE approach, while successfully encoding absolute positional information, has several drawbacks. First, it requires additional learnable parameters which complicates themodel optimization. Second, because the $\mathcal{T}$ matrix is fixed, models trained with this method loses its ability to input context that is longer than the training length.

%% file: section/3_StableMask.tex
\section{StableMask}
\label{sec:stablemask}
In the previous section, we analyzed two problems with the decoder-only Transformer architecture commonly used in contemporary LLMs: disproportional attention and inability to encode absolute position. Disproportional attention happens when certain attention heads share no need to allocate any attention logits but have to due to the softmax mechanism, and this issue is more pronounced at the beginning of the sequence in the decoder. 
The inability to encode absolute position comes from the result of softmax: it is a right stochastic matrix, with the sum of each row equals one always, so its output is insensitive to absolute positions. 

To address the above two problems, we seek a solution by introducing pseudo-attention scores into the softmax operation. Specifically, the solution should simultaneously meet the following requirements:
\begin{enumerate}[label=(\roman*)]
    \item It can provide \textit{additional pseudo-attention scores} to accommodate excess attention logits, thereby freeing DA tokens from the responsibility of absorbing unnecessary attention values. 
    \item These additional pseudo-attention scores need to adhere to the property of DA in a decoder-only model, i.e. larger at the beginning of the sequence and smaller towards the end of the sequence. 
    \item It ensures that the result of softmax is not a right stochastic matrix, i.e. the sum of each row is not 1, so that positional information can be encoded.
\end{enumerate}
In the following section, we show that all of the above three requirements can be met by carefully modifying the causal mask applied after softmax.

\subsection{Pseudo-attention Score}
\label{sec:training}
\setlength{\arraycolsep}{1pt}
To meet the requirement (i) and (ii), we propose constructing a StableMask attention score matrix $A_{\text{SM}} \in \mathbb{R}^{n \times n}$:

\begin{equation}
\setlength{\arraycolsep}{3pt}
    A_{\text{SM}} = \begin{pmatrix}
 a_{11} & p_{11} & \cdots & p_{1 (n-1)} \\
a_{21} & a_{22}  & \cdots & p_{1 (n-2)} \\
\vdots & \vdots &  \ddots & \vdots \\
 a_{n1} & a_{n2} & \cdots & a_{nn} \\
\end{pmatrix}.
\end{equation}
Here, we call these $p_{ij}$ as \textit{pseudo-attention scores}. When the current attention head does not depend too much on its previous context, it can choose to store unnecessary attention values on these pseudo-attention scores. For each row (all attention scores for the $i$-th token), the sequence length it can attend to is fixed to be $n$. Therefore there will be $n-i$ pseudo-attention scores in each row for excessive attention allocation. This fulfills requirement (ii), which involves having more pseudo-attention values towards the beginning of a sequence. $A_{\text{SM}}$ can be calculated using the following method:
\begin{equation}
    A_{\text{SM}} = A \odot C + P,
\end{equation}
\begin{equation*}
\setlength{\arraycolsep}{3pt}
    C = \begin{pmatrix}
1 & 0  & \cdots & 0 \\
1 & 1 & \cdots & 0 \\
\vdots  & \vdots & \ddots & \vdots \\
1 &  1 & \cdots & 1 \\
\end{pmatrix}, P = \begin{pmatrix}
 0 & p_{11} & \cdots & p_{1 (n-1)} \\
0 & 0  & \cdots & p_{1 (n-2)} \\
\vdots & \vdots &  \ddots & \vdots \\
 0 & 0 & \cdots & 0 \\
\end{pmatrix}.
\end{equation*}

The problem then becomes how should the values of these pseudo-attention scores be set. At the start of training, the distribution of the scaled attention scores has a mean of $0$. These attention scores are also influenced by position encoding, and commonly used RPEs typically exhibit decay with increasing relative distance. Therefore, pseudo-attention scores should not significantly disrupt the original distribution of attention scores, and they should also align with the characteristics of the relative position encoding used by the model. Consequently, for $p_{ij}$, it should conform to:
\begin{equation}
    p_{\text{base}} = 0, ~~ p_{ij} = p_{\text{base}} - (j-1) \gamma,
\end{equation}
where $\gamma$ is a decay rate hyperparameter. Therefore, the attention score matrix with StableMask should be:
\begin{equation}
\setlength{\arraycolsep}{3pt}
    A_{\text{SM}} = \begin{pmatrix}
 a_{11} & -\gamma & \cdots & -(n-1)\gamma \\
a_{21} & a_{22}  & \cdots & -(n-1)\gamma \\
\vdots & \vdots &  \ddots & \vdots \\
 a_{n1} & a_{n2} & \cdots & a_{nn} \\
\end{pmatrix}.
\end{equation}
Finally, we can replace the traditional causal mask operation in Equation \eqref{eq:causal_softmax} with:
\begin{eqnarray}
    \tilde{A} &=& \mathrm{Softmax}(A_{\text{SM}}) \odot C \nonumber\\
    &=& \mathrm{Softmax}(A \odot C + P) \odot C.
    \label{eq:remask}
\end{eqnarray}
Here the $A_{\text{SM}}=A \odot C + P$ inside $\mathrm{Softmax}$ masks the attention score matrix with pseudo-attention scores, whereas the $C$ outside $\mathrm{Softmax}$ replaces the scores which need masking with 0 again. Therefore, StableMask still maintains the characteristics of causal decoding, ensuring that information does not leak from subsequent tokens.

\subsection{StableMask Encodes Absolute Position}
StableMask introduces a set of pseudo-attention scores. Therefore, for those real attention scores (the lower triangular part of the attention matrix $A_{\text{SM}}$), their sum after softmax will not be 1, meeting the requirement (iii). Concretely, let $A_i$ denote the real attention scores of the $i$-th row and $P_i$ denote the pseudo-attention scores in the $i$-th row, we have:
\begin{equation*}
    \sum\mathrm{Softmax}_{A_i\bigcup P_i}(A_i) = 1 - \sum\mathrm{Softmax}_{A_i\bigcup P_i}(P_i),
\end{equation*}
where $\mathrm{Softmax}_{A_i\bigcup P_i}(A_i)$ and \(\mathrm{Softmax}_{A_i\bigcup P_i}(P_i)\) are the real/pseudo attention in each row.
We reconsider the question posed in Section \ref{sec:inability}: whether the model can encode positional information for an identical input sequence $X = [\boldsymbol{x}, \cdots, \boldsymbol{x}]_n$. The answer is affirmative:
notice that $\Sigma_{j\leq i}\exp(A_{ij})$ increases as $i$ increases (all \(A_{ij}\)s are equal), and $\Sigma_{j>i}\exp(P_{ij})$ decreases as $i$ increases, we have
\begin{equation*}
     \sum\mathrm{Softmax}_{A_i\bigcup P_i}(A_i) < \sum\mathrm{Softmax}_{A_{i+1}\bigcup P_{i+1}}(A_{i+1}),
\end{equation*}
which means after Equation \(\eqref{eq:remask}\), the output attention values will be monotonic:
\begin{eqnarray*}
    \tilde{A}(W_V X)^{\top} = [\alpha_1 \boldsymbol{v}, \alpha_2 \boldsymbol{v}, \cdots, \alpha_n \boldsymbol{v}]_n, \\
    0 < \alpha_1 < \alpha_2 < \dots < \alpha_n = 1.
\end{eqnarray*}
This indicates that absolute positional information is effectively captured.

In general, a Transformer decoder with StableMask has the ability to encode absolute positional information:
\begin{theorem}
    Let \(X = [\boldsymbol{x}_1, \cdots, \boldsymbol{x}_n]_n\) be an input sequence of length \(n\) to the StableMask model \(f^{\text{(SM)}}_T\). Then, 
    the first layer of \(f^{\text{(SM)}}_T\) can recover absolute positions \([1,2,\dots,n]\) in the hidden state \(\Omega^{(1)}\). That is, there exist \(W_Q\), \(W_K\), \(W_V\) and \(W_O\) for the first attention layer, along with \(W_1\) and \(W_2\) for the first feed-forward layer, that computes absolute positions and pass them to the next layer.
\end{theorem}
The complete proof can be found in Appendix~\ref{appendix:absolute}.


\subsection{Inference and Length Extrapolation}
\label{sec:inference}
\begin{figure}[!ht]
    \centering
    \includegraphics[width=0.9\linewidth]{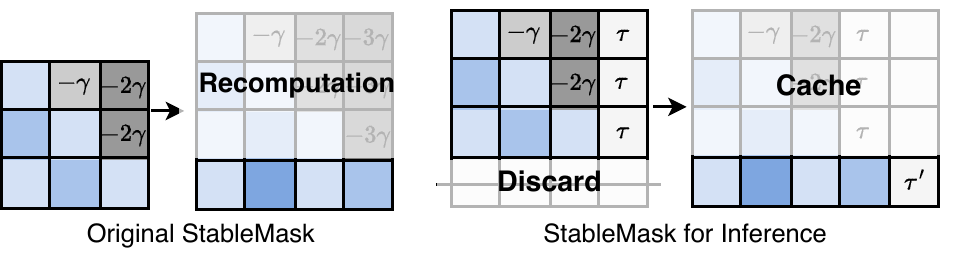}
    \caption{StableMask for Inference. The original StableMask implementation needs to recompute the softmax result for the attention score matrix because additional mask values are added. StableMask for Inference introduces a factor $\tau$ to fix the situation to be in the form of the maximum training length.}
    \label{fig:inference}
\end{figure}

{In Section~\ref{sec:training}}, we introduced the computation process of StableMask. During the training phase, StableMask can be readily applied in parallel within a batch to backpropagate the training loss. 
  During inference, attention computation is usually performed serially and employs KV caching~\cite{tang2021ast,pope2023efficiently}. 

StableMask in its original form is not cost-effective for inference, because it does not support the use of KV caching.
  During the inference stage, when the sequence length is changed e.g. from $n$ to $n+1$ for causal decoding, attention layers need to recalculate the softmax results.
  For the first $n$ rows, an additional pseudo-attention value is added, invalidating the previously calculated attention (see Figure~\ref{fig:inference}).
  This renders KV caching unusable, significantly increasing the cost of inference.
 \begin{table*}[!ht]
    \centering
    \small
    \begin{tabular}{llcc|llccc}
    \toprule
              &\multicolumn{2}{l}{WikiText-103}&& \multicolumn{5}{c}{MiniPile}\\ 
              \midrule
              Model&*PE &\#Params&PPL&Model& *PE & \#Params& PPL 1 Epoch&PPL 2 Epoch\\ 
          \midrule
              BLOOM &ALiBi &71M&$29.9_{\pm .1}$&BLOOM&  ALiBi &160M& $25.8_{\pm .2}$ &$23.3 _{\pm .4}$\\
              BLOOM-SM &ALiBi&71M&$\textbf{29.0}_{\pm .1}$&BLOOM-SM& ALiBi &160M& $\textbf{25.6}_{\pm .0}$&$\textbf{22.9}_{\pm .2}$\\
            \midrule
              OpenLLaMA &RoPE &71M&$27.4_{\pm .2}$&OpenLLaMA&  RoPE&160M& $25.9_{\pm .1}$&$21.2_{\pm .1}$\\
              OpenLLaMA-SM &RoPE &71M&$\textbf{26.9}_{\pm .3}$&OpenLLaMA-SM& RoPE&160M& $\textbf{25.0}_{\pm .0}$&$\textbf{20.9}_{\pm .3}$\\
          \midrule
              BLOOM &ALiBi &160M &$27.6_{\pm .9}$&BLOOM&  ALiBi&430M&  $20.6_{\pm .1}$&$15.6_{\pm .4}$\\
              BLOOM-SM &ALiBi&160M &$\textbf{26.1}_{\pm .2}$&BLOOM-SM&  ALiBi&430M&  $\textbf{19.6}_{\pm .3}$&$\textbf{15.5}_{\pm .2}$\\
            \midrule
              OpenLLaMA &RoPE &160M &$22.5_{\pm .8}$&OpenLLaMA&  RoPE&430M& $19.6_{\pm .2}$&$15.7_{\pm .5}$\\
              OpenLLaMA-SM &RoPE &160M &$\textbf{21.1}_{\pm .6}$&OpenLLaMA-SM& RoPE&430M& $\textbf{19.5}_{\pm .4}$&$\textbf{15.1}_{\pm .5}$\\
         \bottomrule
     \multicolumn{9}{l}{\textit{*: positional encoding type} }\\
    \end{tabular}
    \caption{Pretraining results with (`` -SM'') or without StableMask  on the Wikitext-103 and MiniPile datasets. }
    \label{tab:400M}
\end{table*}
Our solution is simple:
  we pad the sequence to the training length while compressing the padded tokens into a single suffix token.
    Assuming the current sequence length is $n$,
    we first append a suffix token to the end of the sequence (See Figure~\ref{fig:stablemask_arch} (f) and Figure~\ref{fig:inference}).
    At this point, the size of the attention matrix becomes $(n+1) \times (n+1)$. 
    Then, in the additional last column, we add a factor $\tau = \ln(\sum_{i=n}^{N-1}e^{-i\gamma})$:
\begin{equation*}
\label{eq:sm_n+1}
    A'_{\text{SM}} = \begin{pmatrix}
 a_{11} & -\gamma & \cdots & -(n-1)\gamma & \tau \\
a_{21} & a_{22}  & \cdots & -(n-1)\gamma & \tau \\
\vdots & \vdots &  \ddots & \vdots & \vdots \\
 a_{n1} & a_{n2} & \cdots & a_{nn} & \tau\\
 a_{(n+1)1} & a_{(n+1)2} & \cdots & a_{(n+1)n} & a_{(n+1)(n+1)}\\
\end{pmatrix}.
\end{equation*}
  The last row of \(A'_{\text{SM}}\) comes from the suffix and will not be utilized for generation.
  This makes each row equivalent to the case when the sequence length is the same as the training length, allowing us to use KV caching.

Next, we deal with the length extrapolation scenario, i.e. inputs that are longer than the pretraining length limit.
Notice that when $n$ reaches the maximum training length \(N\), $\tau$ becomes $0$. 
This setup prevents the model from continuing to generate $\tau$ values beyond the training length.
Therefore, during extrapolation, we set $\tau=-n\gamma$, where $n\geq N$ is the current sequence length. \(\tau\) in long sequences is a very small number after applying the softmax, and its value will approach zero as \(n\) grows. 
However, the presence of this term still ensures that the softmax result is not a right stochastic matrix, thereby asymptotically encoding absolute positional information.
In addition, when the sequence length is very long, the phenomenon of disproportional attention nearly disappears, as we concluded in Section~\ref{sec:da}. Hence the pseudo-attention score does not need to maintain a large value.

\subsection{Hardware-Efficient Implementation of StableMask}

\flashattn~\cite{dao2022flashattention} represents a major advance in accelerating the Transformer architecture. It avoids repeated data transfers between GPU's High Bandwidth Memory (HBM) and processing units, by segmenting and sequentially processing the \(QKV\) matrix on-chip.
StableMask's integration into this framework is seamless, requiring only minimal modifications. 
In the \flashattn paradigm, the query \( Q \in \mathbb{R}^{n \times d_H} \), key \( K \in \mathbb{R}^{n \times d_H} \), and value \( V \in \mathbb{R}^{n \times d_H} \) matrices are partitioned into \( Tr = \frac{n}{Br} \) blocks \( Q_1, \ldots, Q_{Tr} \), \( K_1, \ldots, K_{Tr} \), \( V_1, \ldots, V_{Tr} \), each of dimension \( \mathbb{R}^{Br \times d_H} \). Then each block \( Q_i, K_j, V_i \) is fetched for computation. The attention scores \( S_i^{(j)} \) for blocks \( Q_i \) and \( K_j \) are derived from the on-chip computation:
    $S_i^{(j)} = Q_i K_j^T \in \mathbb{R}^{Br \times Br}$.
With the incorporation of StableMask into \flashattn, two additional fused operations are introduced as follows:
\begin{equation}
    S_i^{(j)} = (Q_i K_j^T) \odot C_i^{(j)} + P_i^{(j)},
\end{equation}
where \( P \) and \( C \) correspond to the StableMask matrices, segmented into \( Tr \times Tr \) blocks with \( P_i^{(j)}, C_i^{(j)} \in \mathbb{R}^{Br \times Br} \) and loaded on-chip.
We include a complete formula derivation and pseudocode implementation in Appendix \ref{appendix:flash_att}.

%% file: section/4_Experiment.tex
\begin{table*}[ht]
    \centering
    \small
    \begin{tabular}{l|ccccc|cccccc} 
        \toprule
         \multirow{2}{*}{Model}&&\multicolumn{3}{c}{PPL / Tokens}& & \multicolumn{6}{c}{DownStream Tasks}\\
         &5B&10B & 15B& 20B&25B & LBD& PIQA& ARCE&ARCC&OBQA &WG\\ 
          \midrule
          OpenLLaMA&15.4$_{\pm .2}$& 14.8$_{\pm .3}$& 12.4$_{\pm .3}$& 11.7$_{\pm .2}$ &10.7$_{\pm .3}$& 59.4& 67.1&  51.4&25.6& 31.4&53.5\\
         OpenLLaMa-SM& \textbf{15.0}$_{\pm .2}$& \textbf{14.6}$_{\pm .1}$& \textbf{11.9}$_{\pm .1}$& \textbf{11.3}$_{\pm.4}$&\textbf{10.4}$_{\pm.3}$& \textbf{59.6}& 67.1& \textbf{51.7}&25.6& \textbf{32.6}& \textbf{54.1}\\
         \bottomrule
    \end{tabular}
    \caption{Left: Pretraining result of OpenLLaMA 1.4B with RoPE. Right: Result of downstream tasks on OpenLLaMA 1.4B.}
    \label{tab:1B}
\end{table*}

\begin{figure*}[ht]
    \centering
    \includegraphics[width=1\textwidth]{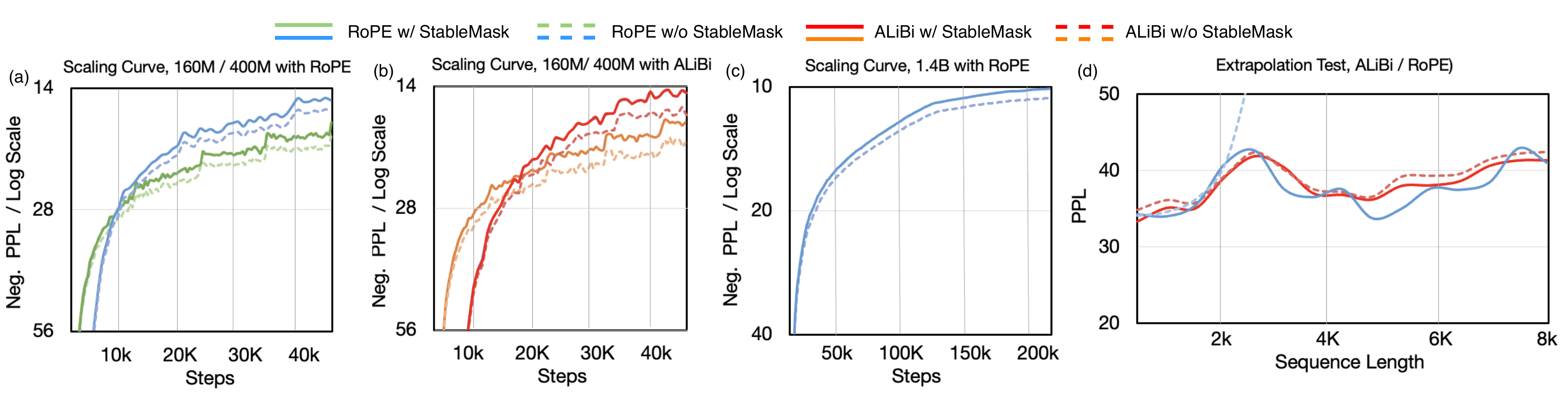}
    \vspace{-0.5cm} 
    \caption{ (abc): Scaling Curve of models from 160M to 1.4B across different positional encodings. (d): extrapolation results (with window attention). StableMask consistently improves the model performance while enabling effective extrapolation.}
    
    \label{fig:scaling_curve}
\end{figure*}

\section{Experiments}
In this section, we present extensive experiments to rigorously evaluate the performance of our proposed method. 

\subsection{StableMask Solves Two Problems}
Our initial assessment confirms the efficacy of the StableMask model in addressing the two problems in Transformer models. The experimental results have been presented in Figure~\ref{fig:da_tokens}. Firstly, concerning the disproportionate attention problem, we perform a comparative visualization of the attention heads in models with and without StableMask. By calculating the attention probability ratios for the first token and various token types, we observed that StableMask largely rectifies the issue of abnormal attention distribution. With StableMark, both initial tokens and punctuation marks experience a significant reduction in attention values. Regarding the second issue of encoding absolute positional information, we evaluated the model's fitting capabilities on a specially designed dataset, comparing StableMask with various Position Encoding approaches. The findings indicate StableMask adeptly encodes absolute positional information, thereby effectively remedying the limitations inherent in relative position encoding. We also provided a visualization of the new attention score matrix after softmax with StableMask in Appendix~\ref{appendix:visualization}.

\subsection{StableMask Improves Model Performance}
We further tested the performance of StableMask on various model architectures and position encodings. Our experiments leverage models built on BLOOM (LLaMA architecture with ALiBi) and OpenLLaMA~\cite{touvron2023llama} (RoPE~\cite{su2021roformer}) architectures. Detail settings could be checked in the Appendix~\ref{appendix:training_details}. 

\textbf{Performance on Wikitext-103 and MiniPile (Table~\ref{tab:400M})}: Empirical evidence underscores the efficacy of models employing StableMask when trained on both Wikitext-103~\cite{merity2016pointer} and MiniPile~\cite{kaddour2023minipile}. These models demonstrate enhanced perplexity (PPL) scores, a pattern consistent across different architectures and sizes, including those with ALiBi and RoPE, and spanning parameter scales of 71M to 400M. Notably, within those datasets, models integrating StableMask consistently outshine their counterparts lacking this feature.

\textbf{Impact on Scaling Performance (Table~\ref{tab:1B})}: The Pile is an extensive open-source dataset tailored for large-scale language modeling. We pretrained a 1.4B model with LLaMA architecture on the Pile dataset with 25B tokens. In the context of scaling of tokens, the model with StableMask consistently achieves better PPL scores compared to the standard OpenLLaMA model, showing the scaling ability of models with StableMask.

\textbf{Effectiveness in Downstream Tasks (Table~\ref{tab:1B})}: When examining pre-trained models on downstream tasks like LAMBADA~\cite{paperno2016lambada}, PIQA~\cite{bisk2019piqa}, ARC-Easy~\cite{Yadav_2019}, ARC-Challenge~\cite{Yadav_2019}, OpenbookQA~\cite{Mihaylov2018CanAS}, and Winogrande~\cite{sakaguchi2021winogrande}, model with StableMask shows a general trend of improved performance. It suggests that StableMask not only improves language understanding in the pretraining stage but also enhances effectiveness in downstream tasks.
\subsection{Extrapolation Capability}
As StableMask resolves the problem of DA tokens, it naturally addresses the attention sink issue~\cite{xiao2023efficient}, where initial tokens get large attention values and removing them from the attention window leads to a surge in perplexity. The models with our proposed StableMask do not need to preserve tokens at the beginning of the sequence during window-based extrapolation and avoid causing generation failures.
As shown in Figure~\ref{fig:scaling_curve}, when using the RoPE position encoding, the extrapolation perplexity quickly explodes without StableMask. When StableMask is applied, the extrapolation perplexity remains stable with window
 attention, where only the most recent KVs are cached.
Furthermore, we believe that the parameter-free nature of StableMask facilitates its seamless integration with other extrapolation methods, a prospect we leave for future exploration.

%

\begin{table}[ht]
    \centering
    \small
    \begin{tabular}{lc|lc}
        \toprule
              Methods& PPL  & Pseudo Value&PPL\\
         \midrule
              Baseline & 22.5 & $-\infty$&22.5\\
              Learnable AT& 21.6 & 0&21.5\\
               Fixed Value AT& 22.4 & $1 \times 10^{-2}$ &22.2\\
              StableMask& \textbf{21.1} & Positional Decay&\textbf{21.1}\\
         \bottomrule
    \end{tabular}
    \caption{Left: Experiment result of ablation study and comparison of AT method on OpenLLaMA, 160M. Right: Ablation experiment, 160M on OpenLLaMA.}
    \label{tab:ablation}
    \vspace{-1em}
\end{table}
\subsection{StableMask vs AT-based Methods}
In Section \ref{sec:two}, we discussed that the artificial token (AT)-based methods are one alternative method to mitigate the DA problem. These artificial tokens could be either learnable, i.e. added before the embedding layer, or fixed as constant vectors, e.g. zero vector. However, we find that as AT-based methods provide the same number of tokens for all sequences, its benefit is not as significant as StableMask (see Table \ref{tab:ablation}) since the severity of the DA issue varies along the sequence. For fair comparison, we retrained OpenLLaMA models using the AT method and StableMask on the MiniPile dataset. 

\subsection{Impact on Inference Efficiency}

In Section~\ref{sec:inference}, we introduced StableMask for Inference, which changes the form of the mask to allow for more efficient inference strategies like KV cache. To validate its effectiveness, we tested the inference efficiency of a standard Transformer (Baseline), a model using StableMask (SM), and a model using StableMask for Inference (SM-I). We present the results in Figure~\ref{fig:lantency} and find that StableMask for Inference significantly improved the model's inference efficiency, making it comparable to the efficiency of traditional Transformers.

\begin{figure}[ht]
    \centering
    \includegraphics[width=0.85\linewidth]{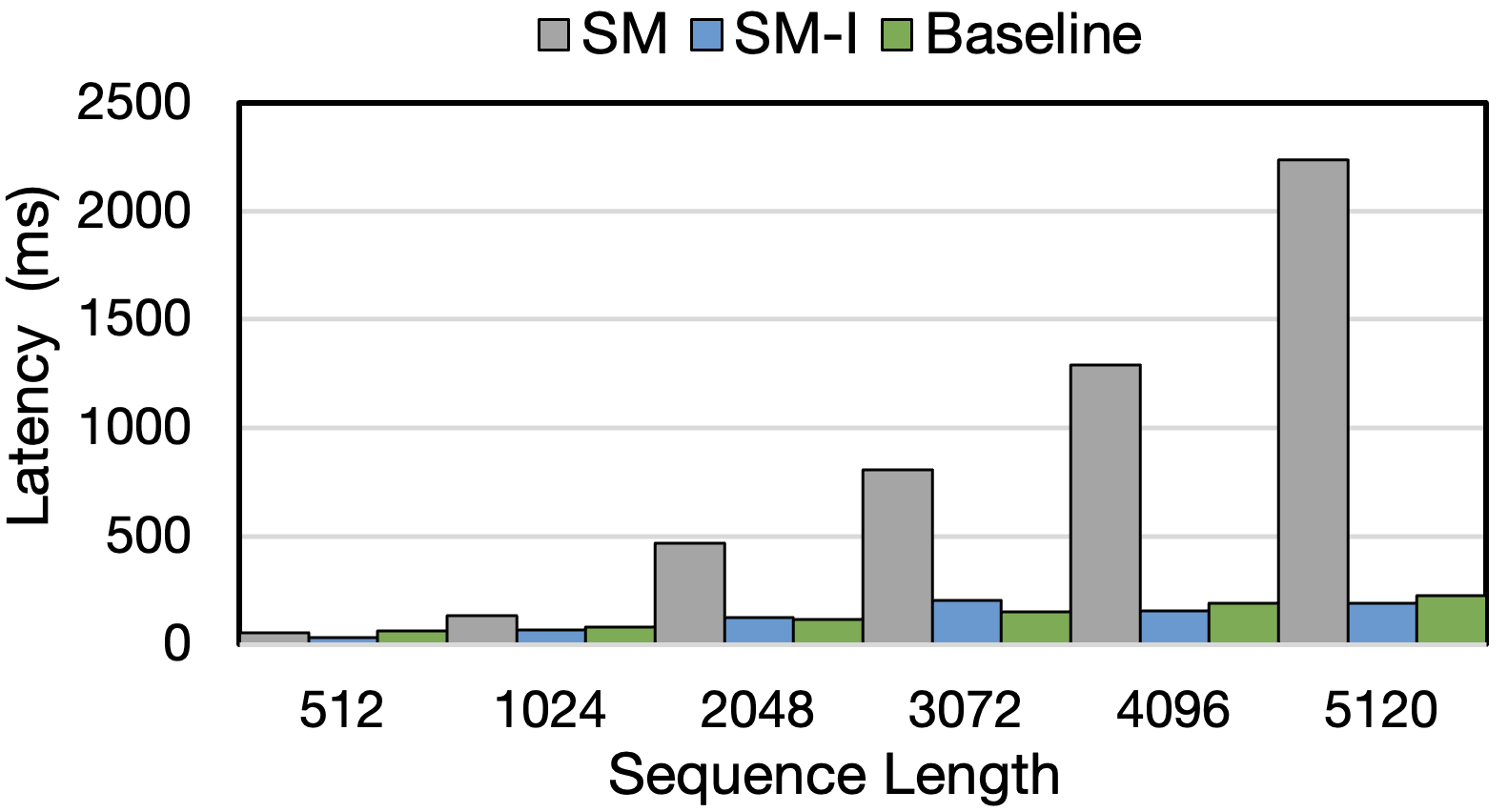}
    \caption{Inference latency test on OpenLLaMA 1.4B. Our proposed StableMask adapted for fast inference (SM-I) significantly reduces the running latency.}
    \label{fig:lantency}
\end{figure}

\subsection{Effects of Pseudo Attention Value}
In Section \ref{sec:stablemask}, we introduced positional linear decay, making the pseudo-attention scores align with the characteristics of real attention scores.w
To validate its rationality, we conducted ablation experiments on various types of pseudo-attention scores. These experiments included four modes: (a) No addition of pseudo-attention scores, i.e., maintaining a mask of negative infinity.
(b) Padding with zeros, which aligns with the values of attention score distribution.
(c) Padding with a value different from the attention score distribution, e.g. $1\times 10^{-2}$.
(d) The positional decay method we proposed. 

Our ablation studies, as detailed in Table \ref{tab:ablation}, demonstrate that a decay value like \(1 \times 10^{-2}\) deviates significantly from the original attention matrix's distribution, leading to diminished pretraining performance. The implementation of positional decay, however, excels in the training phase, showcasing state-of-the-art performance.

%% file: section/5_Conclusion.tex
\section{Related Work}
\label{sec:related}
Several studies have attempted to address issues inherent in the attention mechanism and softmax operation. A pivotal contribution by \cite{hassid2022does} raised questions about the role of certain heads in the attention mechanism. They discovered that substituting a subset of heads with constant diagonal matrices could even enhance model performance, suggesting that part of the model's attention heads do not need to attend to any tokens other than themselves. 
Quantizable Transformer~\cite{bondarenko2023quantizable} and StreamingLLM~\cite{xiao2023efficient} identified a tendency in some attention heads to accumulate probabilities on the initial few tokens or on tokens similar to punctuation marks. \citet{bondarenko2023quantizable} demonstrated that this behavior impacts model quantization, proposing a solution by trimming softmax and employing gated attention. StreamingLLM, on the other hand, observed that this phenomenon affects windowed attention, and addressed it by preserving the initial tokens.
~\citet{darcet2023vision} proposed adding ``register tokens'' which are essentially artificial places for the real tokens to attend to. The added tokens serve as a way to absorb the excessive attention that would otherwise accumulate on the initial tokens.

However, the previous approach of adding or using extra tokens 
either (1) uses fixed values or weights which does not account for possible distributional shifts when extrapolating to longer sequences; (2) does not explore its potential interference with positional embeddings; (3) adds extra parameters or computation to the attention layer, while not making clear whether existing optimization techniques are still applicable; (4) does not provide a theoretical framework for understanding the phenomenon more deeply.

\section{Conclusion}
StableMask represents a significant advancement in the field of language modeling, by simultaneously addressing two limitations of the decoder-only Transformer architecture: disproportional attention and inability to encode absolute position. By refining the causal mask with pseudo-attention values, StableMask adeptly balances attention distributions and encodes absolute positional information through a progressively decreasing mask ratio. It preserves the inherent distribution of the attention score matrix and enhances the model's ability in various natural language tasks. 

While StableMask demonstrates much potential, it is not without its constraints. One notable limitation is the slightly increased computational demand compared to conventional attention mechanisms. However, as the increased computation is only one matrix multiplication, we believe this overhead is negligible. Furthermore, StableMask inherently encodes absolute positional information, necessitating careful calibration to prevent the model from being adversely affected. We anticipate that forthcoming research will further refine our approach and overcome these challenges.

\section{Acknowledgement}

We thank Songlin Yang and other collaborators for the suggestions on language expression and image design in this paper.



%% file: section/6_Appendix.tex
\section{Detailed Explanation of the DA Issue}
\label{appendix:dispro}
\newcommand{\Inp}{X} 
\newcommand{\inp}{x} 
\newcommand{\Hid}{\Omega}
\newcommand{\Hidl}{\Omega^{(l)}}
\newcommand{\HidL}{\Omega^{(L)}}
\newcommand{\hid}{\omega}
\newcommand{\hidl}{\omega^{(l)}}
\newcommand{\hidL}{\omega^{(L)}}
\newcommand{\Out}{X_{n+1}} 
\newcommand{\Vocab}{\mathcal{V}}
\newcommand{\embdX}{X_E}
\newcommand{\NoPE}{}
The traditional dot-product attention makes the assumption that the next token is strongly related to the previous context. However, the mutual information \(I(\Inp_{\leq i}; \Out)=H(\Out) - H(\Out | \Inp_{\leq i})\) could be small, especially in the initial parts of the sequence. We formalize this (counter-)intuition by defining the following concepts:

\begin{definition} A \emph{causally isotropic} data distribution of \(N\) discrete random variables \(\Inp_1, \Inp_2, \dots, \Inp_N\) satisfies that for any set of indices \(\Lambda\subset [n]\), \(H(\Inp_{n+1} | \Inp_{\Lambda} = \inp_{\Lambda}) = H(\Inp_{n+1} | \Inp_{\Lambda})\) does not depend on the value of \(\inp_{\Lambda}\), where \(H\) denotes entropy\footnote{Causal isotropy is a strict condition. We use it for demonstration purposes only: it isolates the effect of data variability in judging the disproportionality of attention.}.
\end{definition}
\begin{definition} A \emph{layer-wise} decoder for a data distribution \(p(\Inp_1, \Inp_2, \dots, \Inp_N)\) accepts any data point \(\inp_{<N}\), and computes deterministically \(L\) layers of intermediate representations \(\Hidl_{<N}\), such that for \(n < N\), \(\Hidl_{n}\) only receives inputs from \(\Hid^{(l-1)}_{\leq n}\) (we define \(\Hid^{(0)}_{n}\) as \(\Inp_n\) or its embedding).
\end{definition}
\begin{definition} An \emph{contextual} layer-wise decoder satisfies that for any two possible inputs \(\inp_{<N}, \inp'_{<N}\) and \(n<N\), if \(p(\Inp_{n+1}|\inp_{\leq n}) \neq p(\Inp_{n+1}|\inp'_{\leq n})\), then \(\hidL_{n} \neq \hid'^{(L)}_{n}\), where \(\hidL_{ n}\) (\(\hid'^{(L)}_{ n}\)) is \(\HidL_{n}\) evaluated on input \(\inp_{\leq n}\) (\(\inp'_{\leq n}\)).
\end{definition}
Our definition of contextual decoder aligns with the definition of contextual mapping in previous works~\cite{yun2020transformers,kim2023provable}, which guarantees that certain different inputs are mapped to different representations, although their definition of contextual mapping is more focused on the seq2seq setting.

Next, we make the following observations:
\begin{proposition}\label{appendix:proposition}
For a layer-wise decoder on a data distribution, the prefixes of its intermediate representation at the \(l\)-th layer \(\Hidl_{\leq i}\) satisfy 
\begin{enumerate}
    \item \(H(\Out | \Hidl_{\leq i}) \geq H(\Out | \Inp_{\leq i})\) for all \(i\leq n\);
    \item \(H(\Out | \Hidl_{\leq n}) = H(\Out | \Inp_{\leq n})\) if the decoder is contextual;
    \item \(H(\Out | \Hidl_{\leq i} = \hid^{(l)}_{\leq i}) \geq H(\Out | \Inp_{\leq i})\) for all \(i\leq n\) and all \(\hid^{(l)}_{\leq i}\) if the data is causally isotropic;
    \item \(H(\Out | \Hidl_{ \leq n} = \hid^{(l)}_{ \leq n}) = H(\Out | \Inp_{\leq n})\) for all \(\hid^{(l)}_{ \leq n}\) if the decoder is contextual and the data is causally isotropic.
\end{enumerate}
\end{proposition}
\begin{proof} \begin{enumerate}
    \item Notice that \(\Out \rightarrow \Inp_{\leq i} \rightarrow \Hidl_{\leq i}\) is a Markov chain. By the data processing inequality~\cite{polyanskiy2016strong}, \(I(\Hidl_{\leq i};\Out) \leq I(\Inp_{\leq i}; \Out)\implies H(\Out | \Hidl_{\leq i}) \geq H(\Out | \Inp_{\leq i})\).
    \item For any \(\hidl_{\leq n}\),
let \(\kappa(\hidl_{\leq n})\) be the set of inputs where \(p(\inp_{\leq n}|\hidl_{\leq n}) > 0\), which is equivalent to \(p(\hidl_{\leq n}|\inp_{\leq n}) = 1\) by the deterministic nature of decoder.

By the definition of contextual layer-wise decoder,
\begin{align}
    \forall \inp_{\leq n}, \inp'_{\leq n}\in \kappa(\hidl_{\leq n}) &\implies \hidl_{\leq n} = \hid'^{(l)}_{\leq n} \implies \hidL_{\leq n} = \hid'^{(L)}_{\leq n} \nonumber\\ &\implies \hidL_{n} = \hid'^{(L)}_{n} \implies p(\Out|\inp_{\leq n}) = p(\Out|\inp'_{\leq n}).
\end{align}
Therefore,
\begin{align}
    p(\Out|\hidl_{\leq n}) &= \sum_{\inp_{\leq n}\in \kappa(\hidl_{\leq n})} p(\Out|\inp_{\leq n}, \hidl_{\leq n}) p(\inp_{\leq n}|\hidl_{\leq n}) \nonumber\\
\text{(by conditional independence)}~~~ &= \sum_{\inp_{\leq n}\in \kappa(\hidl_{\leq n})} p(\Out|\inp_{\leq n}) p(\inp_{\leq n}|\hidl_{\leq n}) \label{eq:cond_indep}\\
&= p(\Out|\inp_{\leq n}), ~\forall \inp_{\leq n}\in \kappa(\hidl_{\leq n}) \nonumber\\
    \implies H(\Out|\Hidl_{\leq n}=\hidl_{\leq n}) &= H(\Out|\Inp_{\leq n}=\inp_{\leq n}), ~\forall \inp_{\leq n}\in \kappa(\hidl_{\leq n}) \label{eq:entropy_equal}
\end{align}
\begin{align}
\implies H(\Out|\Hidl_{\leq n}) &= \sum_{\hidl_{\leq n}} p(\hidl_{\leq n}) H(\Out|\Hidl_{\leq n}=\hidl_{\leq n}) = \sum_{\hidl_{\leq n}} \left(\sum_{\inp_{\leq n}\in \kappa(\hidl_{\leq n})} p(\inp_{\leq n})\right) H(\Out|\Hidl_{\leq n}=\hidl_{\leq n}) \nonumber\\
&= \sum_{\inp_{\leq n}} p(\inp_{\leq n}) H(\Out|\Inp_{\leq n}=\inp_{\leq n}) = H(\Out|\Inp_{\leq n}).
\end{align}
    \item
Note that \eqref{eq:cond_indep} can be written as a weighted average, which we denote as \(\mathrm{avg}_{\kappa}\):
\begin{align}
p(\inp_{n+1}|\hidl_{\leq n})
    &= \mathrm{avg}_{\kappa} p(\inp_{n+1}|\inp_{\leq n}),~\forall \inp_{n+1}.
\end{align}
Similarly, with a slightly different definition of \(\kappa\),
\begin{align}
p(\inp_{n+1}|\hidl_{\leq i})
    &= \mathrm{avg}_{\kappa} p(\inp_{n+1}|\inp_{\leq i}), \forall x_{n+1}.
\end{align}
Apply Jensen's inequality to the function \(-x \log x\), we have for any \(x_{n+1}\),
\begin{align}
-p(\inp_{n+1}|\hidl_{\leq i}) \log p(\inp_{n+1}|\hidl_{\leq i})
&= -\left(\mathrm{avg}_{\kappa} p(\inp_{n+1}|\inp_{\leq i})\right) \log \left(\mathrm{avg}_{\kappa} p(\inp_{n+1}|\inp_{\leq i})\right) \nonumber\\
&\geq \mathrm{avg}_{\kappa} \left(-p(\inp_{n+1}|\inp_{\leq i}) \log p(\inp_{n+1}|\inp_{\leq i})\right). \nonumber\\
\end{align}
Therefore,
\begin{align}
    H(\Out|\Hidl_{\leq i} = \hidl_{\leq i}) &= -\sum_{\inp_{n+1}} p(\inp_{n+1}|\hidl_{\leq i}) \log p(\inp_{n+1}|\hidl_{\leq i}) \nonumber\\
    &\geq\sum_{\inp_{n+1}} \mathrm{avg}_{\kappa}\left(-p(\inp_{n+1}|\inp_{\leq i}) \log p(\inp_{n+1}|\inp_{\leq i})\right) \nonumber\\
    &=\mathrm{avg}_{\kappa} \sum_{\inp_{n+1}} -p(\inp_{n+1}|\inp_{\leq i}) \log p(\inp_{n+1}|\inp_{\leq i}) \nonumber\\
    &=\mathrm{avg}_{\kappa} H(\Out|\Inp_{\leq i} = \inp_{\leq i}) \nonumber\\
\text{(by causal isotropy)}~~~ &=H(\Out|\Inp_{\leq i}).
\end{align}
\item Apply causal isotropy to \eqref{eq:entropy_equal}.
\end{enumerate}
\end{proof}
\begin{figure*}[ht]
    \centering
    \includegraphics[width=0.6\textwidth]{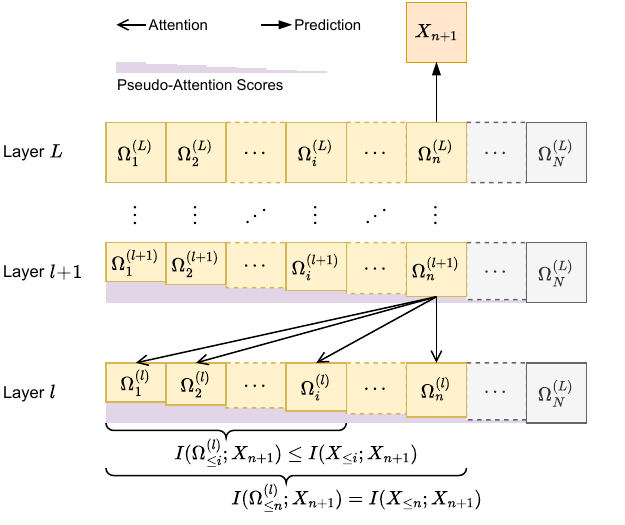}
    \caption{The DA issue and the proposed solution of adding pseudo-attention scores. The rationale behind is that through learning, a decoder should learn to avoid paying too much attention to where \(H(\Out|\Hidl_{\leq i} = \hidl_{\leq i})\) is high, because such places provide little mutual information with respect to the prediction goal.}
    \label{fig:DA_appendix}
\end{figure*}
We are now ready to define the disproportionality of attention:
\begin{definition} Let inputs sampled from a data distribution \(p(\Inp_1, \Inp_2, \dots, \Inp_N)\) run through a contextual layer-wise decoder with attention layers. If for at least one possible input \(\inp_{<N}\), the attention \(\tilde{A}^{(l)}\) after softmax in the \(l\)-th layer satisfy
\begin{equation}    
\sum_{j\leq i} \tilde{A}^{(l)}_{nj} > \frac{I(\Inp_{\leq i};\Inp_{n+1})}{I(\Inp_{\leq n};\Inp_{n+1})}\sum_{j\leq n} \tilde{A}^{(l)}_{nj} + \varepsilon
\end{equation}
for some \(i < n < N\) and \(I(\Inp_{\leq n};\Inp_{n+1})>0\), then this attention layer is said to have disproportional attention towards initial tokens on this input. The overall degree of disproportionality of an attention layer can be measured by the total probability of such inputs \(\sum_{\inp_{<N}} p(\inp_{<N})\).
\end{definition}
Note that by Proposition~\ref{appendix:proposition}, the following always holds:
\begin{equation}
    \frac{I(\Inp_{\leq i};\Inp_{n+1})}{I(\Inp_{\leq n};\Inp_{n+1})} = \frac{H(\Out) - H(\Out|\Inp_{\leq i})}{H(\Out) - H(\Out|\Inp_{\leq n})} \geq \frac{H(\Out) - H(\Out|\Hidl_{\leq i})}{H(\Out) - H(\Out|\Hidl_{\leq n})} = \frac{I(\Hidl_{\leq i};\Inp_{n+1})}{I(\Hidl_{\leq n};\Inp_{n+1})}.
\end{equation}
This justifies our choice of the threshold \(\frac{I(\Inp_{\leq i};\Inp_{n+1})}{I(\Inp_{\leq n};\Inp_{n+1})}\) for detecting the disproportionality of attention. Moreover, if the data is causally isotropic, the specific values of data do not matter for how much attention the model should pay.

In this work, we handle the DA problem by pseudo-attention scores and we offer a probabilistic interpretation. First, we clarify that the problem does not lie in the query-key-value mechanism of attention, but rather lies in the nature of autoregression: the history does not represent a complete description of the future, and the probability that the future deviates from the history must be taken into account, and more so at the beginning. Thus the output of an attention layer at earlier positions should be able to signal to the subsequent layers a higher variance of estimation compared to later positions. The failure of reliably doing so leads to the model having to allocate computation elsewhere to rectify the signal, such as excessive attention towards irrelevant tokens~\cite{xiao2023efficient} and ``no-op'' heads~\cite{bondarenko2023quantizable}, or becoming totally paralyzed (Appendix~\ref{appendix:edge-case}). StableMask parameterizes this inductive bias orthogonal to decoder-only Transformers with RPE by pseudo-attention scores in the causal mask that decays over time.

\section{Further Explanation of Position Encoding}
\subsection{The Unit Test of Absolute Position-Awareness}
\label{appendix:edge-case}
Training a decoder-only Transformer with no PE will fail on data points that consist of all identical tokens, because the outputs of each layer are all identical vectors. Consequently, it is impossible for the model to predict different output distributions at different positions. We regard such all-identical inputs with different outputs at different positions as the ``unit test'' of absolute position awareness. We showed that Transformers with RPE cannot pass this test (Appendix~\ref{sec:experiment_rpe}).

One way to pass the test without using explicit PE was proposed, by prepending a special \(\langle bos \rangle\) token to the input sequence~\cite{kazemnejad2023impact}. It breaks the symmetry in all positions and provides a way for the decoder to recognize absolute position. We note that this solution is equivalent to the AT-based method used to solve the DA issue (Section~\ref{sec:da}). This inspires us to see the test from the viewpoint of DA. Indeed, we have
\begin{theorem}
There exists a causally isotropic data distribution (defined in Appendix~\ref{appendix:dispro}) such that any regular Transformer decoder has a high probability of being (weakly) disproportional in all of its attention layers.
\end{theorem}
\begin{proof}
    Consider the following \(\mathrm{softCopyLast}\) task: for any input \(x_{<n}\), output the last token with probability \(1 - e^{-n}\), or a random token otherwise. The training dataset is constructed by a sampling algorithm that correctly does the task repeatedly.
    
    The training dataset is causally isotropic: for every set of observed variables \(x_{\Lambda}\), \(\Lambda\subset [n]\), \(H(\Out|\Inp_{\Lambda}=\inp_{\Lambda})\) depends only on the largest element of \(\Lambda\), not on the specific values of variables.

    Moreover, the probability density of this dataset concentrates most on the all-identical sequences, because as time goes on, sequences in the dataset are increasingly likely to copy themselves.

    Last, we need to check that regular Transformer decoders have (weakly) disproportional attention on all identical sequences in all the attention layers. Note that although \(I(\Inp_{\leq i};\Inp_{n+1}) > 0\) for \(i<n\), \(I(\Inp_{\leq i}; \Out | \Inp_{n}) = 0\) holds because of conditional independence between \(\Inp_{\leq i}\) and \(\Out\) given \(X_n\). On the other hand, \(\sum_{j\leq i} \tilde{A}^{(l)}_{nj} > \varepsilon\) holds because the softmax in a regular Transformer always gives positive attention. So the model has a weak disproportional attention towards initial tokens.
\end{proof}
Intuitively speaking, if the inputs are all identical, then the model only needs to know the last token and the sequence length in order to decide the output. All other attention can be regarded as (weakly) disproportional. However, inputs constructed this way only account for an exponentially small total probability in real datasets, so we separate this issue from the issue of disproportional attention.

\subsection{Experiment of RPE's Inability to Encode Absolute Position}
\label{sec:experiment_rpe}
To demonstrate that RPE cannot encode absolute positional information as discussed in Section~\ref{sec:inability}, we designed several experiments that require knowledge of absolute positional relationships. These experiments primarily include three tasks:

\begin{enumerate}[label=(\arabic*)]
    \item Absolute Position Mapping: Given an input sequence of ``\(0~0~0~0~0~\ldots\)'', the model needs to accurately map each position to its absolute position. In other words, we expect an output of ``\(1~2~3~4~5~\ldots\)''.
    \item Absolute Position Identification: Given an input sequence of ``\(0~0~0~\ldots~\texttt{[ABE]}~0~0~\ldots\)'', where \texttt{[ABE]} encodes a special character at a specific position, the model needs to output the absolute position corresponding to the location encoded by \texttt{[ABE]}. In this case, we expect an output of ``\(0~0~0~\ldots~n~0~0~\ldots\)'', where \(n\) represents the current position.
    \item Odd-Even Number Counting: Given an input sequence of ``\(0~0~0~0~0~\ldots\)'', the model needs to output a sequence of consecutive odd and even numbers, such as ``\(1~2~1~2~\ldots\)''. This task also relies on the model's ability to recognize absolute positional information.
\end{enumerate}

\begin{table}[ht]
    \centering
    \small
    \begin{tabular}{l|ccc}
        \midrule
         & \multicolumn{3}{c}{Accuracy}\\
        \midrule
         PE*& Task (1)& Task (2)& Task (3)\\
         \midrule
         \multicolumn{4}{l}{\textit{{APE}}}\\
         \midrule
         Learnable& \textbf{96.7}\%& \textbf{94.3}\%&\textbf{97.6}\%\\
         Sinusoidal& \textbf{98.1}\%& \textbf{99.1}\%& \textbf{96.2}\%\\
         \midrule
         \multicolumn{4}{l}{\textit{{RPE}}}\\
         \midrule
         ALiBi & 21.7\%& 26.7\%& 46.5\%\\
         T5& 22.4\%& 24.5\%&42.7\%\\
         RoPE & 25.3\%& 24.7\%& 43.1\%\\
         \bottomrule
 \multicolumn{4}{l}{\textit{*: positional encoding type}}\\
    \end{tabular}
    \caption{
    Experiment settings and Results of RPE's inability to encode absolute position.
    We designed three datasets that rely on absolute position information and calculated average accuracy on these tasks. The results show that RPE performs poorly.  This demonstrates that the position information encoded during the softmax process in RPE is shadowed.
    }
    \label{tab:position}
\end{table}

Our experiments were conducted using a model with 160 million parameters, trained on four V100 GPUs. For detailed training hyperparameters, one can refer to the training details on the Wikitext-103 dataset (Appendix~\ref{appendix:training_details}).

\section{StableMask Encodes Absolute Positional Information}
\label{appendix:absolute}
In this section, we present how StableMask can recover absolute positions in the hidden state using fewer portions of the model, than prepending a special \(\langle bos \rangle\) token to the sequence (Appendix~\ref{appendix:edge-case}). Our proof is inspired by NoPE~\cite{kazemnejad2023impact} but differs substantially in that they require three dimensions of hidden states at free disposal, while ours only needs two and is arguably more natural.
\begin{theorem}
    Let \(X = [\boldsymbol{x}_1, \cdots, \boldsymbol{x}_n]_n\) be an input sequence of length \(n\) to the StableMask model \(f^{\text{(SM)}}_T\). Then, 
    the first layer of \(f^{\text{(SM)}}_T\) can recover absolute positions \([1,2,\dots,n]\) in the hidden state \(\Omega^{(1)}\). That is, there exist \(W_Q\), \(W_K\), \(W_V\) and \(W_O\) for the first attention layer, along with \(W_1\) and \(W_2\) for the first feed-forward layer, that computes absolute positions and pass them to the next layer.
\end{theorem}
\begin{proof}
We focus on the goal of reconstructing an index-dependent function \(\xi_{i}=i / (i + \sum_{j=i}^{n-1} e^{-j\gamma})\) at the end of the first attention layer. After reconstructing \(\xi_{i}\), recovering \(i\) from it can be done by the universal approximation power of feed-forward networks~\cite{park2021minimum}.

For this, we need to gain control of a single head in the first attention layer, and use two hidden dimensions in the embedding layer. Note that this approach does not alter the rest of the Transformer model.

First, we specify the word embedding matrix $W_E \in \mathbb{R}^{d \times \Vocab}$ as follows: the first row of $W_E$ is set to 1, which serves as the input vector; The second row of $W_E$ is set to 0, which serves as the output vector. Then, we have:
\begin{equation}
    \setlength{\arraycolsep}{3pt}
    W_E = \begin{pmatrix}
            1       & 1       & \dots  & 1       \\
            0       & 0       & \dots  & 0       \\
            e_{3,1} & e_{3,2} & \dots  & e_{3,\Vocab} \\
            \vdots  & \vdots  & \ddots & \vdots  \\
            e_{d,1} & e_{d,2} & \dots  & e_{d,\Vocab}
          \end{pmatrix}_{d \times \Vocab}
\end{equation}
where $e_{i,j} \in \mathbb{R}$. The word embeddings for the input sequence $X = [x_1, \dots, x_n]_n$ are retrieved from the embedding matrix $W_E$ by:
\begin{equation}
    \setlength{\arraycolsep}{3pt}
    \embdX = W_E[X] = \begin{pmatrix}
                        1       & 1       & \dots  & 1       \\
                        0       & 0       & \dots  & 0       \\
                        e_{3,x_1} & e_{3,x_2} & \dots  & e_{3,x_n} \\
                        \vdots  & \vdots  & \ddots & \vdots  \\
                        e_{d,x_1} & e_{d,x_2} & \dots  & e_{d,x_n}
                      \end{pmatrix}_{d \times n}
\end{equation}
Second, for head dimension $h\ge 1$, we specify the weights $W_Q, W_K, W_V, W_O$ of the selected attention head in the first layer. Specifically, we set \(W_Q=W_K=0\), and
\begin{equation}
    \setlength{\arraycolsep}{3pt}
    W_V = \begin{pmatrix}
              1      & 0      & \dots  & 0      \\
              0      & 0      & \dots  & 0      \\
              \vdots & \vdots & \ddots & \vdots \\
              0      & 0      & \dots  & 0
            \end{pmatrix}_{h \times d},\quad
    W_O = \begin{pmatrix}
              0      & 0      & \dots  & 0      \\
              1      & 0      & \dots  & 0      \\
              \vdots & \vdots & \ddots & \vdots \\
              0      & 0      & \dots  & 0
    \end{pmatrix}_{d\times h}.
\end{equation}
Consequently, all the query-key matching results are zero:
\begin{equation}
W_K \embdX = W_Q \embdX = 0_{h\times n},\quad A = (W_Q \embdX)^\top(W_K \embdX) = 0_{n\times n},    
\end{equation}
while $W_V$ takes the first row of \(X_E\), which is the input vector, and sets everywhere else zero:
\begin{equation}
    \setlength{\arraycolsep}{3pt}
    W_V \embdX = \begin{pmatrix}
              1      & 1      & \dots  & 1      \\
              0      & 0      & \dots  & 0      \\
              \vdots & \vdots & \ddots & \vdots \\
              0      & 0      & \dots  & 0
            \end{pmatrix}_{h \times n}
\end{equation}
We now calculate the output of attention. First, since the key-query matching results are all zero, the attention score matrix with StableMask is
\begin{equation}
    \setlength{\arraycolsep}{3pt}
    A_{\text{SM}} = A \odot C + P = \begin{pmatrix}
         0 & -\gamma & \cdots & -(n-1)\gamma \\
         0 & 0 & \cdots & -(n-1)\gamma \\
        \vdots & \vdots & \ddots & \vdots \\
         0 & 0 &\cdots & 0 \\
    \end{pmatrix}_{n\times n}
\end{equation}
Therefore,
\begin{equation}
    \setlength{\arraycolsep}{3pt}
    \tilde{A} = \mathrm{Softmax}(A_{\text{SM}}) \odot C = \begin{pmatrix}
         1 / (1 + \sum_{i=1}^{n-1} e^{-i\gamma}) & 0 & \cdots & 0 \\
         1 / (2 + \sum_{i=2}^{n-1} e^{-i\gamma}) & 1 / (2 + \sum_{i=2}^{n-1} e^{-i\gamma})  & \cdots & 0 \\
        \vdots & \vdots &  \ddots & \vdots \\
         1 / n & 1 / n & \cdots & 1 / n \\
    \end{pmatrix}_{n\times n}
\end{equation}
\begin{equation}
    \setlength{\arraycolsep}{3pt}
    \tilde{A} (W_V X_E)^\top = \begin{pmatrix}
         1 / (1 + \sum_{i=1}^{n-1} e^{-i\gamma}) & 0 & \cdots & 0 \\
         2 / (2 + \sum_{i=2}^{n-1} e^{-i\gamma}) & 0  & \cdots & 0 \\
         \vdots & \vdots &  \ddots & \vdots \\
         1 & 0 & \cdots & 0 \\
    \end{pmatrix}_{n\times h} = \begin{pmatrix}
         \xi_1 & 0 & \cdots & 0 \\
         \xi_2 & 0  & \cdots & 0 \\
         \vdots & \vdots &  \ddots & \vdots \\
         \xi_n & 0 & \cdots & 0 \\
    \end{pmatrix}_{n\times h}
\end{equation}
Finally, \(W_O\) is used to move the first row of \((\tilde{A} (W_V X_E)^\top)^\top\) to the second row:
\begin{equation}
    \setlength{\arraycolsep}{3pt}
    W_O (\tilde{A} (W_V X_E)^\top)^\top = \begin{pmatrix}
         0 & 0 & \cdots & 0 \\
         \xi_1 & \xi_2  & \cdots & \xi_n \\
         \vdots & \vdots &  \ddots & \vdots \\
         0 & 0 & \cdots & 0 \\
    \end{pmatrix}_{d\times n}
\end{equation}
Adding the residuals back to the input, we are done:
\begin{equation}
    \setlength{\arraycolsep}{3pt}
    X_E + \sum_\mathrm{h} W^{(\mathrm{h})}_O (\tilde{A}^{(\mathrm{h})} (W^{(\mathrm{h})}_V X_E)^\top)^\top = \begin{pmatrix}
         1 & 1 & \cdots & 1 \\
         \xi_1 & \xi_2  & \cdots & \xi_n \\
         * & * & \cdots & * \\
         \vdots & \vdots &  \ddots & \vdots \\
         * & * & \cdots & * \\
    \end{pmatrix}_{d\times n}
\end{equation}
where \(*\) denotes values computed by other heads in the first layer, which we assumed to not interfere with the first two hidden dimensions.
\end{proof}

\section{\flashattn with StableMask}
\label{appendix:flash_att}
\subsection{Introduction to \flashattn}
\flashattn~\cite{dao2022flashattention} is a state-of-the-art method designed to enhance the performance of attention mechanisms in Transformer models, particularly addressing the efficiency constraints imposed by modern GPU memory hierarchies. Traditional attention mechanisms suffer from significant computational overhead, predominantly due to the necessity of storing and accessing large intermediate matrices, such as the softmax-normalized attention scores, from the High Bandwidth Memory (HBM). This process is inherently memory-bound due to the quadratic dependency on the sequence length, leading to extensive memory accesses and thus increased wall-clock time.

The A100 GPU, for instance, showcases the discrepancy in memory speeds within its hierarchy, having a significantly faster on-chip SRAM compared to the larger HBM. \flashattn optimizes for this architectural detail by reducing HBM reads and writes. It achieves a sub-quadratic number of HBM accesses by employing techniques like tiling and recomputation, which allow for the attention computation to be performed in smaller, more manageable blocks within the on-chip SRAM. This block-based approach mitigates the need to store large intermediate matrices, especially beneficial during the backward pass of model training where intermediate values are traditionally saved to HBM.

Furthermore, \flashattn incorporates kernel fusion in its implementation, enabling a single CUDA kernel to handle the entire computation process -- from loading inputs from HBM, through all the computation steps (such as matrix multiplication and softmax), to writing the results back to HBM. This minimizes the frequency of costly memory accesses and contributes to an overall faster computation, without compromising the accuracy of the attention mechanism. As a result, \flashattn stands out as an efficient primitive for both memory-bound and compute-bound operations within the GPU's memory hierarchy, offering a significant improvement in the execution of Transformer models.

\subsection{Derivation}

In the \flashattn paradigm, the query \( Q \in \mathbb{R}^{n \times d_H} \), key \( K \in \mathbb{R}^{n \times d_H} \), and value \( V \in \mathbb{R}^{n \times d_H} \) matrices are partitioned into \( Tr = \frac{n}{Br} \) blocks \( Q_1, \ldots, Q_{Tr} \), \( K_1, \ldots, K_{Tr} \), \( V_1, \ldots, V_{Tr} \), each of dimension \( \mathbb{R}^{Br \times d_H} \). Then each block \( Q_i, K_j, V_i \) is fetched for computation. The attention scores \( S_i^{(j)} \) for blocks \( Q_i \) and \( K_j \) are derived from the on-chip computation:
    $S_i^{(j)} = Q_i K_j^T \in \mathbb{R}^{Br \times Br}$.
With the incorporation of StableMask, two additional on-chip operations are introduced:
\begin{equation}
    S_i^{(j)} = (Q_i K_j^T) \odot C_i^{(j)} + P_i^{(j)},
\end{equation}
where \( P \) and \( C \) correspond to the StableMask matrices, segmented into \( Tr \times Tr \) blocks with \( P_i^{(j)}, C_i^{(j)} \in \mathbb{R}^{Br \times Br} \), and loaded on-chip. The safe softmax operation, analogous to that in \flashattn, proceeds as follows:
\begin{eqnarray}
    m_i^{(j)} &=& \max(m_i^{(j-1)}, \mathrm{rowmax}(S_i^{(j)})) \in \mathbb{R}^{Br},\\
    \tilde{S_i^{(j)}} &=& \exp(S_i^{(j)} - m_i^{(j)}) \in \mathbb{R}^{Br \times Br}, \\
    l_i^{(j)} &=& e^{m_i^{(j)} - m_i^{(j-1)}} l_i^{(j-1)} + \mathrm{rowsum}(\tilde{S_i^{(j)}}) \in \mathbb{R}^{Br}.
\end{eqnarray}
Subsequently, the algorithm rectifies the attention score matrix to account for zeros necessitated by the causal mask, so the final output \( O_i^{(j)} \) is computed as:
\begin{equation}
    O_i^{(j)} = \mathrm{diag}(e^{m_i^{(j)} - m_i^{(j-1)}})^{-1} O_i^{(j-1)} + (\tilde{S_i^{(j)}} \odot C_i^{(j)}) V_i.
\end{equation}

\subsection{A Typical Implementation of \flashattn 2}

\begin{algorithm*}[!h]
  \caption{\small Forward pass}
  \begin{algorithmic}[1]
    \REQUIRE Matrices $\vQ, \vK, \vV, \mathbf{C}, \mathbf{B} \in \mathbb{R}^{N \times d}$ in HBM, block sizes $B_c$, $B_r$.
    \STATE \label{alg:stream_attn_split_qkv} Divide $\vQ$ into $T_r = \left\lceil\frac{N}{B_r} \right\rceil$ blocks $\vQ_1, \dots, \vQ_{T_r}$ of size $B_r \times d$ each,
    and divide $\vK, \vV$ in to $T_c = \left\lceil \frac{N}{B_c} \right\rceil$ blocks $\vK_1, \dots, \vK_{T_c}$ and
    $\vV_1, \dots, \vV_{T_c}$, of size $B_c \times d$ each. 
    \textcolor{purple}{Divide $\mathbf{C}, \mathbf{P}$ in to $T_r \times T_c  = \left\lceil \frac{N}{B_r} \right\rceil \times \left\lceil \frac{N}{B_c} \right\rceil$ blocks $\mathbf{C}_1, \dots, \mathbf{C}_{T_r}$ and
    $\mathbf{P}_1, \dots, \mathbf{P}_{T_c}$, of size $B_r \times B_c$ each.}
    \STATE Divide the output $\vO \in \mathbb{R}^{N \times d}$ into $T_r$ blocks $\vO_i, \dots, \vO_{T_r}$ of size
    $B_r \times d$ each, and divide the logsumexp $L$ into $T_r$ blocks $L_i, \dots, L_{T_r}$ of size
    $B_r$ each.
    \FOR{$1 \le i \le T_r$} \label{alg:stream_attn_outer_loop}
      \STATE \label{alg:stream_attn_load_q} Load $\vQ_i$ from HBM to on-chip SRAM.
      \STATE \label{alg:stream_attn_init} On chip, initialize $\vO_{i}^{(0)} = (0)_{B_r \times d} \in \mathbb{R}^{B_r \times d}, \ell_{i}^{(0)} = (0)_{B_r} \in \mathbb{R}^{B_r}, m_{i}^{(0)} = (-\infty)_{B_r} \in \mathbb{R}^{B_r}$.
      \FOR{$1 \le j \le T_c$}
        \STATE \label{alg:stream_attn_load_kv} Load $\vK_j, \vV_j$ \textcolor{purple}{$\mathbf{C}_i^{(j)}, \mathbf{P}_i^{(j)}$} from HBM to on-chip SRAM.
        \STATE \label{alg:stream_attn_qk} On chip, compute $\vS_{i}^{(j)} = \vQ_i \vK_j^T \textcolor{purple}{\odot \mathbf{C}_i^{(j)} + \mathbf{P}_i^{(j)}} \in \mathbb{R}^{B_r \times B_c}$.
        \STATE \label{alg:stream_attn_statistics} On chip, compute
        $m_{i}^{(j)} = \mathrm{max}(m_{i}^{(j-1)}, \mathrm{rowmax}(\vS_{i}^{(j)})) \in \mathbb{R}^{B_r}$, $\tilde{\vP}_{i}^{(j)} = \exp(\vS_{i}^{(j)} - m_{i}^{(j)}) \in \mathbb{R}^{B_r \times B_c}$ (pointwise),
        $\ell_{i}^{(j)} = e^{m_{i}^{j-1} - m_{i}^{(j)}} \ell_{i}^{(j-1)} + \mathrm{row sum}(\tilde{\vP}_{i}^{(j)}) \in \mathbb{R}^{B_r}$.
        \STATE  \textcolor{purple}{On chip, compute $\tilde{\mathbf{D}}_{i}^{(j)} \tilde{\vP}_{i}^{(j)} \odot \mathbf{C}_i$}.
        \STATE \label{alg:stream_attn_update} On chip, compute
        $\vO_{i}^{(j)} = \diag(e^{m_{i}^{(j-1)} - m_{i}^{(j)}})^{-1} \vO_{i}^{(j-1)} + \textcolor{purple}{{\mathbf{D}}_{i}^{(j)}} \vV_j$.
      \ENDFOR
      \STATE On chip, compute $\vO_{i} = \diag(\ell_{i}^{(T_c)})^{-1} \vO_{i}^{(T_c)}$.
      \STATE On chip, compute $L_{i} = m_{i}^{(T_c)} + \log(\ell_i^{(T_c)})$.
      \STATE Write $\vO_{i}$ to HBM as the $i$-th block of $\vO$.
      \STATE Write $L_{i}$ to HBM as the $i$-th block of $L$.
    \ENDFOR
    \STATE Return the output $\vO$ and the logsumexp $L$.
  \end{algorithmic}
\end{algorithm*}
(The parts that are different from the original algorithm are marked in purple.)

\newpage
\section{Training Details}
\label{appendix:training_details}

\begin{table}[ht]
    \centering
    \small
    \begin{tabular}{l|cccc}
    \toprule
         & & & &\\
         \midrule
         Parameters&   71M& 160M&400M &1.4B\\
 Embedding Size& 512& 768& 1024&2048\\
         Hidden Size (Attention)&   512& 1536& 2048&4096\\
         Hidden Size (FFN)&   2048& 3072& 2048&8192\\
         Expanding Rate (FFN)&   4& 4& 2&4\\
         Activation Function&   SwishGeLU& SwishGeLU& SwishGeLU&SwishGeLU\\
         Normalization Type& RMSNorm& RMSNorm& RMSNorm&RMSNorm\\
         Positional Encoding&   RoPE / ALiBi& RoPE / ALiBi& RoPE / ALiBi&RoPE\\
         Tokenizer&   GPT2 Tokenizer& GPT2 Tokenizer& GPT2 Tokenizer&GPT2 Tokenizer\\
 Vocabulary Size& 50257& 50257& 50257&50257\\
         \# of Attention Heads&8& 12& 16&16\\
         \# of Layers&6& 12& 16&24\\
    \bottomrule
    \end{tabular}
    \caption{Hyperparameters for WikiText-103 with ALibi and RoPE positional encoding}
\end{table}

\begin{table}[!h]
    \centering
    \small
    \resizebox{\linewidth}{!}{%
    \begin{tabular}{l|c||l|c||l|c}
    \toprule
         \multicolumn{2}{c}{Hyperparameters for Wikitext-103} & \multicolumn{2}{c}{Hyperparameters for MiniPile} & \multicolumn{2}{c}{Hyperparameters for the Pile}\\
         \midrule
         Data&  WikiText-103 & Data&MiniPile & Data&Pile\\
         Sequence Length&  512 & Sequence Length&512 / 1024 & Sequence Length&1024\\
         Batch Size&  64 & Batch Size&128 & Batch Size&128\\
         Tokens per Batch&  32768 & Tokens per Batch&65536 / 131072 & Tokens per Batch&131072\\
         Total Steps&  50k & Steps per Epoch&22k & Total Steps&200k \\
         Warmup Steps&4k & Total Epoch&2 & Warmup Steps &4k\\
         Beginning Learning Rate&  1e-6 & Warmup Steps &4k & Beginning Learning Rate&5e-6\\
         Peak Learning Rate&  6e-4 & Beginning Learning Rate&1e-6 & Peak Learning Rate&2e-4\\
         Learning Rate Decay&Linear & Peak Learning Rate&4e-4 & Learning Rate Decay &Cosine\\
         Optimizer&AdamW & Learning Rate Decay &Linear & Optimizer&AdamW\\
         Adam $\epsilon$ &$1 \times 10^{-8}$ & Optimizer&AdamW & Adam $\epsilon$ &$1 \times 10^{-8}$\\
         Adam $\beta_1$&0.9 & Adam $\epsilon$ &$1 \times 10^{-8}$ & Adam $\beta_1$&0.9\\
         Adam $\beta_2$&0.98 & Adam $\beta_1$&0.9 & Adam $\beta_2$&0.98\\
         Hidden Dropout&0 & Adam $\beta_2$&0.98 & Hidden Dropout&0\\
         GELU Dropout&0 & Hidden Dropout&0 & GELU Dropout&0\\
         Attention Dropout (if needed)&0 & GELU Dropout&0 & Attention Dropout (if needed)&0\\
         Weight Decay&0.01 & Attention Dropout (if needed)&0 & Weight Decay&0.1\\
         Gradient Clipping Value&1 & Weight Decay&0.1 & Gradient Clipping Value&1\\
         Head-wise $\gamma$&True & Gradient Clipping Value&1 & Head-wise $\gamma$&True\\
         $\gamma$ Value&0.5 & Head-wise $\gamma$&True & $\gamma$ Value&0.5\\
    \bottomrule
    \end{tabular}
    }
    \caption{Hyperparameters for WikiText-103 with ALibi and RoPE positional encoding}
\end{table}
\section{{Visualization of Attention Heads with StableMask}}
\label{appendix:visualization}
\newpage
\includepdf[pages={1, 3, 4, 5, 6, 7, 8}]{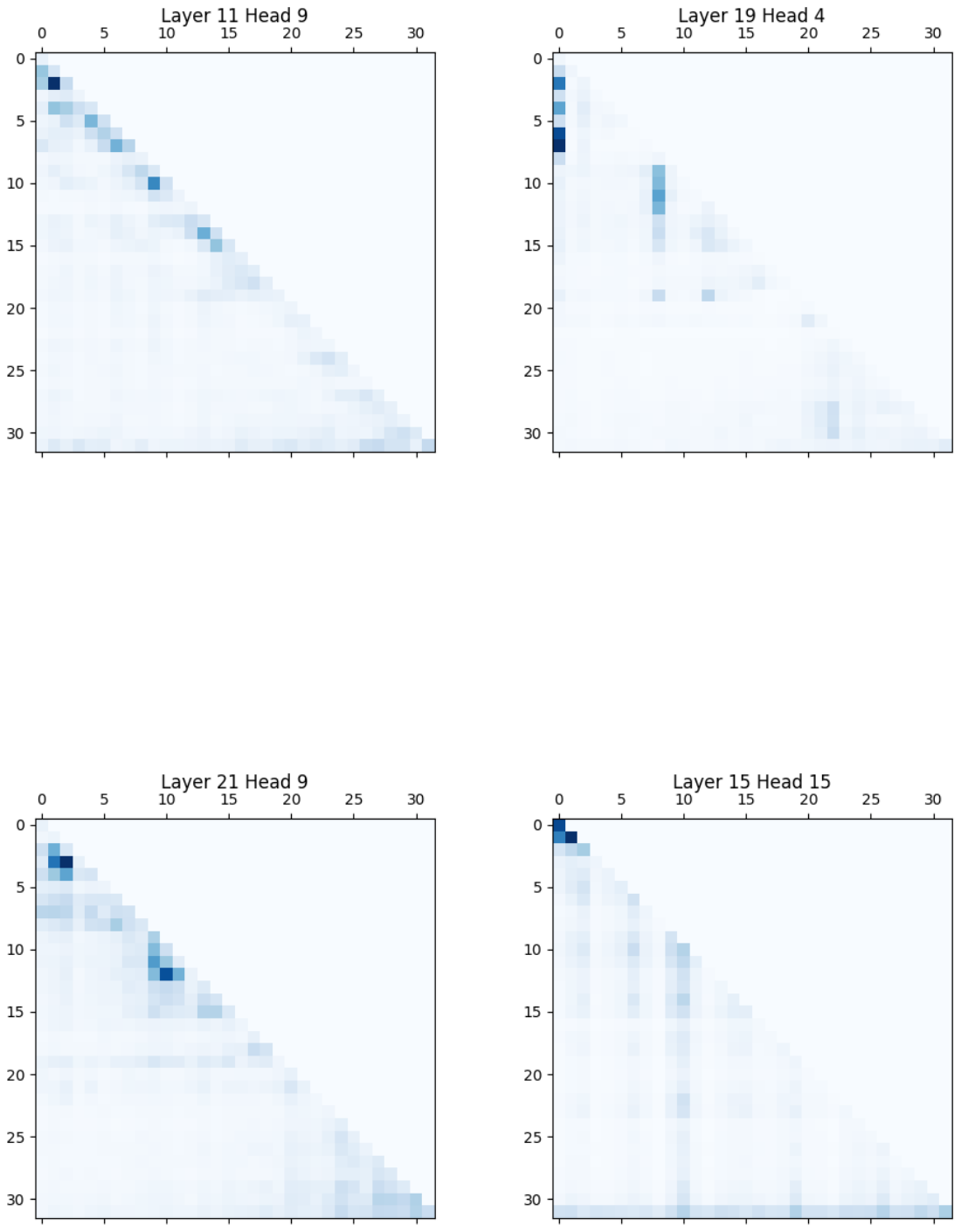}